\documentclass[11pt]{article}
\usepackage{amsthm}
\usepackage{amsfonts}
\usepackage{xcolor}
\usepackage{mathtools}
\usepackage{enumitem}
\usepackage{hyperref}
\usepackage{geometry}
\usepackage{natbib}

\usepackage{thmtools}
\usepackage{thm-restate}

\usepackage{algorithm}
\usepackage{algorithmic}

\newcommand{\authnote}[2]{}

\newcommand{\Enote}[1]{\authnote{Eliad}{#1}}

\def\cA{{\cal A}}

\def\cD{{\cal D}}

\def\cH{{\cal H}}

\def\cP{{\cal P}}

\def\cT{{\cal T}}

\def\cX{{\cal X}}
\def\cY{{\cal Y}}

\def\bU{\mathbf{U}}
\def\bV{\mathbf{V}}
\def\bS{\mathbf{S}}
\def\bP{\mathbf{P}}
\def\bt{\mathbf{t}}

\def\bh{\mathbf{h}}

\def\bbF{{\mathbb F}}

\newcommand{\Tableofcontents}{
\thispagestyle{empty}
\pagenumbering{gobble}
\clearpage
\tableofcontents
\thispagestyle{empty}
\clearpage
\pagenumbering{arabic}
}

\newcommand{\set}[1]{\left\{#1\right\}}
\newcommand{\zo}{\set{0,1}}

\newcommand{\remove}[1]{}

\newcommand{\RepLinearSpan}{\texttt{RepLinearSpan}}

\newcommand{\eps}{\varepsilon}

\newcommand{\la}{\leftarrow}

\newcommand{\ip}[1]{\iprod{#1}}
\newcommand{\iprod}[1]{\langle #1 \rangle}

\newtheorem{claim}{Claim}
\newtheorem{theorem}{Theorem}
\newtheorem{lemma}{Lemma}
\newtheorem{proposition}{Proposition}

\newtheorem{remark}{Remark}

\newtheorem{corollary}{Corollary}
\newtheorem{definition}{Definition}[section]

\title{Computationally Efficient Replicable Learning of Parities and Applications}

\author{Moshe Noivirt\thanks{Department of Computer Science, Johns Hopkins University. E-mails: {\tt noivirt.moshe@gmail.com}, {\tt jess@jhu.edu}.}
\and
Jessica Sorrell$^\ast$
\and
Eliad Tsfadia\thanks{Department of Computer Science, Bar-Ilan University. E-mail: {\tt eliad.tsfadia@biu.ac.il}.}
}

\begin{document}
\maketitle

\begin{abstract}%

    We study the computational relationship between replicability (Impagliazzo et al. [STOC `22], Ghazi et al. [NeurIPS `21]) and other stability notions.  
    Specifically, we focus on replicable PAC learning and its connections to differential privacy (Dwork et al. [TCC 2006]) and to the statistical query (SQ) model (Kearns [JACM `98]). Statistically, it was known that differentially private learning and replicable learning are equivalent and strictly more powerful than SQ-learning. Yet, computationally, all previously known efficient (i.e., polynomial-time) replicable learning algorithms were confined to SQ-learnable tasks or restricted distributions, in contrast to differentially private learning. 

    Our main contribution is the first computationally efficient replicable algorithm for realizable learning of parities over arbitrary distributions, a task that is known to be hard in the SQ-model, but possible under differential privacy.
    This result provides the first evidence that efficient replicable learning over general distributions strictly extends efficient SQ-learning, and is closer in power to efficient differentially private learning, despite computational separations between replicability and privacy.
    Additionally, we leverage our parity learner to prove that, assuming $RP \neq NP$, converting replicability to pure differential privacy requires a strict loss in sample complexity.
    Our main building block is a new, efficient, and replicable algorithm that, given a set of vectors, outputs a subspace of their linear span that covers most of them.
    
\end{abstract}

\Tableofcontents

\section{Introduction}

Stable algorithms have numerous uses in learning theory, giving provable guarantees of generalization~\citep{bousquet2002stability,dwork2015preserving,cummings2016adaptive}, privacy~\citep{dwork2006calibrating}, robustness~\citep{asi2023robustness,hopkins2023robustness}, and replicability~\citep{ghazi2021user,impagliazzo2022reproducibility}. While distinct stability notions have been introduced to capture each of these desiderata, many of these notions have been shown to be statistically equivalent, making stability a unifying principle in the study of trustworthy machine learning~\citep{moran2023bayesian,asi2023robustness,hopkins2023robustness,bassily2016algorithmic}. These equivalences are not all computational, however, and efficient reductions between stability notions is an active area of research.

In this work, we make progress on understanding the computational relationships between privacy, replicability, and the statistical query (SQ) model. Previously, a statistical separation between replicability and the SQ model was demonstrated via the heavy hitters problem~\citep{impagliazzo2022reproducibility}. They gave a replicable algorithm with sample complexity independent of the domain size of the problem instance, and a lower bound on the number of queries required by any SQ algorithm that depends logarithmically on the domain size. Since the dependence is only logarithmic, this result falls short of demonstrating a \emph{SQ-hard} problem admitting a computationally efficient replicable algorithm, and therefore a full computational separation between replicability and the SQ model. Some progress on computationally separating replicability and the SQ model was made in the same work~\citep{impagliazzo2022reproducibility}, by considering replicable algorithms for parities, a problem known to be SQ-hard~\citep{Kearns98}. They observed that parities could be efficiently replicably learned over the uniform distribution, as there is a unique solution in this case. Notably, the implied algorithm (Gaussian elimination) is not replicable for other distributions, and therefore does not satisfy the standard definition of replicability which is universally quantified over distributions. More significant progress followed in~\cite{kalavasis2024computational}, where they gave a replicable algorithm for parities that runs in time polynomial in all relevant parameters of the instance, but exponential in the decision tree complexity of the target distribution. These works left open the question of computationally efficient replicable algorithms for parities over arbitrary distributions. 

Turning to the relationship between replicability and privacy, these stability notions are known to be statistically equivalent~\citep{ghazi2021user,bun2023stability,kalavasis2023statistical}, i.e., the existence of a sample-efficient algorithm for a statistical task satisfying one notion implies the existence of a sample-efficient algorithm satisfying the other. The computational picture, however, is not as clear. Efficient replicable algorithms imply efficient private ones, but ~\cite{bun2023stability} show that the converse cannot hold if rerandomizable public-key cryptography (e.g., the Goldwasser-Micali cryptosystem) exists. Furthermore, they show that cryptographic assumptions are necessary for their separation, proving that if one-way functions do not exist, a private algorithm can always be efficiently transformed into a replicable one via correlated sampling. Whether the existence of one-way functions is sufficient for the separation, or if public-key cryptography is required, remains an interesting open problem. 
Whether there exists a useful characterization of learning tasks for which efficient private learning algorithms exist, but replicable algorithms imply cryptographic adversaries also remains open. We know, for instance, that any learning task admitting an efficient statistical query (SQ) algorithm admits an efficiently replicable one, by the replicable statistical query procedure of~\cite{impagliazzo2022reproducibility}. And, to the best of our knowledge, all PAC learning tasks that are known to have efficient replicable algorithms also admit efficient SQ algorithms. One may then wonder if SQ learnable classes are exactly the classes which are computationally efficiently learnable under both privacy and replicability. This further motivates the question of whether efficient replicable algorithms exist for parities, since efficient private algorithms are known for this problem even under the strong notion of pure differential privacy (where the $\delta$ parameter of $(\varepsilon, \delta)$-DP is taken to be 0)~\citep{kasiviswanathan2011can}. 

In this work, we give the first efficient algorithm for replicable realizable learning of parities, and therefore the first efficient algorithm for replicable PAC learning of a class known to be SQ hard~\citep{Kearns98}. We restrict ourselves to the realizable setting out of necessity: it is known that agnostic learning of parities is NP-hard~\citep{haastad2001some}, and learning parities with constant noise rate smaller than 1/2 is conjectured to be cryptographically hard~\citep{blum2003noise}. Our algorithm makes use of a replicable subspace identification subroutine, which finds a subspace of the span of the support of a target distribution capturing at least a $1-\varepsilon$ fraction of the distribution.
 
Our algorithm has additional implications for the relationship between replicability and pure differential privacy. It is already known that there is no generic transformation from an efficient replicable algorithm to an efficient pure DP algorithm, since the latter has sample complexity characterized by Representation dimension~\citep{beimel2019characterizing}, and there exist efficiently replicably learnable classes with infinite representation dimension (e.g. point funtions). However, if we restrict to classes with representation dimension polynomial in the Littlestone dimension -- that is, classes for which efficient transformations from replicable algorithms to pure private algorithms may plausibly exist -- our replicable parity learner shows that any such transformation cannot preserve the sample complexity dependence on the failure probability. This follows from the optimal additive dependence of our algorithm on its failure probability ($O(\log(1/\delta)/\varepsilon)$), and known lower bounds on the sample complexity of pure private learning of parities, following from connections to robustness~\cite{georgiev2022privacy}.

\subsection{Our Results}

Our main result is a computationally efficient and replicable algorithm for learning parity functions. Namely, it has polynomial running time, and it satisfies the following two properties:

\begin{enumerate}
    \item \textbf{Realizable $(\eps,\delta)$-PAC Learning:} Given i.i.d.\ samples $x_1,\ldots,x_m \in \mathbb{F}_2^d$ drawn from an unknown distribution, and labels $y_1,\ldots,y_m \in \bbF_2$ such that $y_i = \ip{x_i,z}$ for some unknown $z \in \mathbb{F}_2^d$, the algorithm, with probability $1-\delta$, outputs $w \in \mathbb{F}_2^d$ that correctly predicts the label of a fresh sample with probability $1-\eps$.

    \item \textbf{$\rho$-Replicability:} When executing the algorithm on two independent input sets using the same internal randomness, the two executions will output the same $w$ with probability at least $1-\rho$.
\end{enumerate}

\begin{theorem}[Replicable Learning of Parities]\label{thm:intro:parities}
    There exists a polynomial-time $\rho$-replicable learning algorithm that (realizably) $(\eps,\delta)$-PAC learns the class of parity functions over $\bbF_2^d$ with sample complexity $poly(d,1/\rho,1/\eps,\log(1/\delta))$.
\end{theorem}

Our main building block for proving Theorem~\ref{thm:intro:parities} is a new replicable algorithm that, given a set of input vectors over $\bbF^d$ (for an arbitrary field $\bbF$), outputs a subspace of their linear span that covers most of them.

\begin{theorem}[Replicable Linear Span]\label{thm:intro:linearspan}
    There exists a polynomial-time algorithm $\RepLinearSpan$ that is $\rho$-replicable, and given input vectors $v_1,\ldots,v_m \in \mathbb{F}^d$ for $m \geq poly(d,1/\rho,1/\eps)$, outputs a subspace $V \subseteq Span\set{v_1,\ldots,v_m}$ that covers $1-\eps$ fraction of the vectors.
\end{theorem}

Theorem~\ref{thm:intro:linearspan} implies that if the $m$ input vectors are drawn i.i.d.\ from a distribution $\cD$ over $\mathbb{F}^d$, then by standard generalization theorems, with probability at least $1-\delta$, the output subspace $V$ also covers $\approx 1-\eps$ fraction of the distribution $\cD$ (or equivalently, covers at least $1-\eps$ fraction of the distribution using slightly higher sample complexity).

\begin{corollary}
    Given $m \geq poly(d,1/\rho,1/\eps, \log(1/\delta))$ i.i.d.\ samples from an unknown distribution $\cD$, algorithm $\RepLinearSpan$ outputs a subspace $V$ such that, with probability at least $1-\delta$, $\Pr_{v \sim \cD}[v \in V] \geq 1-\eps$.
\end{corollary}

Finally, we leverage our replicable parity learner to reveal a fundamental limitation in algorithmically converting replicable algorithms to pure differentially private algorithms. Our replicable parity learner achieves sample complexity dependency on the failure parameter, scaling as $O(\log(1/\delta)/\varepsilon)$. By contrasting this with known pure DP lower bounds for parity learning \cite{georgiev2022privacy}, we establish the following corollary.

\begin{restatable}{corollary}{repdpimpos}
Let $\mathcal{C}$ denote the set of all classes for which computationally efficient pure private algorithms exist. Assuming $\text{RP} \neq \text{NP}$, there is no \remove{does not exist a}generic polynomial-time transformation that converts an arbitrary efficient replicable algorithm for a class in $\mathcal{C}$ into an efficient pure DP algorithm while preserving the sample complexity dependence on the failure probability.
\end{restatable}

\subsection{Paper Structure}
We provide an overview of our main techniques in Section \ref{sec:tech}. In Section \ref{sec:related_work}, we provide a discussion of additional related work. Technical preliminaries used throughout the paper appear in Section \ref{sec:prelim}.
In Section~\ref{sec:linspan}, we present our replicable linear span algorithm, and in Section~\ref{sec:learning_parities}, we use it to construct our replicable parity learner. In Section \ref{sec:conversion_impossibility}, we leverage our replicable parity learner to establish an impossibility result regarding conversions to pure differential privacy.

Missing proofs appear in Appendices~\ref{sec:linear_span_proof}, \ref{sec:learning_parities_proof}, \ref{sec:main_lem}, and \ref{sec:adversarial_distribution}.
Appendix \ref{sec:empirical} contains our empirical evaluation, and in Appendix \ref{sec:stable_partition}, we recall the stable partition algorithm from \cite{kaplan2025differentially}, which our replicable linear span algorithm uses as a subroutine.
\section{Technical Overview}\label{sec:tech}

In this section, we provide overviews of our proofs. We start with an overview of the replicable linear span algorithm (Theorem~\ref{thm:intro:linearspan}). For simplicity, we focus here on $\eps = \rho = 0.1$. 
Namely, our goal is to design an algorithm $\cA$ that given input vectors $S = (v_1,\ldots,v_m) \in (\mathbb{F}^d)^m$, outputs a subspace $V \subseteq Span\set{v_1,\ldots,v_m}$ that covers $0.9$ fraction of the vectors in $S$, while guaranteeing replicability: For any distribution $\cD$ over $\mathbb{F}^d$, if we draw two independent i.i.d.\ datasets $S_1, S_2 \sim \cD^m$ and internal random coins $r$ for $\cA$, we obtain that $\cA(S_1;r) = \cA(S_2,r)$ w.p. at least $0.9$.

Without replicability restriction, we can use the deterministic algorithm $\cA(S) = Span(S)$. In fact, when we restrict our attention to the binary field $\bbF_2$, this algorithm does provide replicability when $\cD$ is the uniform distribution over $\mathbb{F}_2^d$, because in that case, when $m \geq O(d)$, then it holds w.h.p.\ over $S \sim \cD^m$ that $Span(S) = \mathbb{F}_2^d$. Indeed, this property is used by \cite{impagliazzo2022reproducibility} to learn parities with replicability that is restricted to the uniform distribution. But in general, this algorithm is not replicable, even for vector spaces over the binary field. For example, let $V \subset \bbF_2^d$ be a subspace of dimension strictly less than $d$, let $u \in \bbF_2^d \setminus V$, and consider the distribution $\cD_m$ that draws a uniform element in $V$ w.p.\ $1-1/m$, and $u$ w.p.\ $1/m$. Clearly, the algorithm that given $m$ vectors and outputs their linear span is not $0.9$-replicable because, when the vectors are i.i.d.\ according to $\cD_m$, the algorithm outputs $V$ w.p. $(1-1/m)^m \approx 0.37$ and $Span(V,u)$ otherwise. 
One might hope to achieve replicability by using subsets of the input data, with the subset sizes chosen at random, to avoid distributions that are ``bad'' for specific values of $m$. Yet, we show that there exists a single distribution $\cD$ that is ``bad'' for every $m$, yielding that such approaches are inherently non-replicable, regardless of the choice of subset sizes (the proof appears in Appendix~\ref{sec:adversarial_distribution}).

\begin{proposition}
\label{prop:span_instability}
Let $d \ge 4$ be a perfect square. There exists a distribution $\mathcal{D}$ over $\mathbb{F}_2^d$, such that for any sample size  $ m \le 2^{\sqrt{d}-1}$,  for two independent samples $S_1, S_2 \sim \mathcal{D}^m$, the probability that they generate the exact same linear span decays exponentially with $\sqrt{d}$. Furthermore, the dimension of their sum exceeds the dimension of their intersection by $\Omega(\sqrt{d})$ with constant probability. 
\end{proposition}

\remove{
We prove Proposition~\ref{prop:span_instability} using the following distribution $\cD$ over the standard basis $\{e_1, \dots, e_d\}$ of $\mathbb{F}_2^d$:  We partition  $\{e_1, \dots, e_d\}$ into $\sqrt{d}$ disjoint blocks $B_1, \dots, B_{\sqrt{d}}$, each containing exactly $\sqrt{d}$ vectors. We define the distribution $\mathcal{D}$ as follows: For each block index $j \in \{1, \dots, \sqrt{d}\}$ and for each basis vector $e_k \in B_j$, we assign the probability mass
\[
    \mathcal{D}(e_k) = \frac{1}{2^j \sqrt{d}}.
\]
The remaining probability mass is assigned to the zero vector.

For instance, consider the following distribution over the standard basis $\{e_1, \dots, e_d\}$ of $\mathbb{F}_2^d$. We partition  $\{e_1, \dots, e_d\}$  into $\sqrt{d}$ disjoint blocks $B_1, \dots, B_{\sqrt{d}}$, each containing exactly $\sqrt{d}$ vectors. We define the distribution $\mathcal{D}$ as follows: For each block index $j \in \{1, \dots, \sqrt{d}\}$ and for each basis vector $e_k \in B_j$, we assign the probability mass
\[
    \mathcal{D}(e_k) = \frac{1}{2^j \sqrt{d}}.
\]
The remaining probability mass is assigned to the zero vector.  The following proposition shows that for any polynomial sample size $m = \text{poly}(d)$, the linear spans of two independent samples drawn from $\mathcal{D}$ will diverge with high probability. 
}

To overcome this fundamental instability, we leverage the Stable Partition algorithm introduced by \cite{kaplan2025differentially}. Rather than using a fixed or randomized partition (using a data-independent method), the Stable Partition algorithm iteratively extracts linearly independent subsets from the data sequence. In each pass, it greedily accumulates an independent subset, and then removes this subset from the sequence. It repeats the process until all vectors are assigned, and then outputs the resulting partition.



\paragraph{Replicable Linear Span (Theorem \ref{thm:intro:linearspan})}
To use the stable partition for finding a replicable subspace, we track the multiplicity $n_S(V)$, the number of sets in the partition spanning $V$. Formally,

\begin{definition}
Given a sequence of $m$ vectors $S \in (\bbF^d)^m$, we define the multiplicity of a subspace $V \subseteq \bbF^d$ by 
\begin{align*}
    n_{S}(V) := | \set{ j : \mathrm{span}(A_j) = V} |,
\end{align*}
where $\set{A_1,\ldots,A_M}$ is the stable partition of $S$. 
\end{definition}

Our strategy filters out "light" subspaces (those with low multiplicity) to isolate a replicable set of candidate subspaces. We achieve this by comparing the multiplicities to a random threshold $t$, chosen uniformly from a large interval. 

The core challenge in this approach is ensuring replicability: for two independent executions, one with input $S_1$ and the other with input $S_2$, to compute the same set of candidate subspaces, the shared threshold $t$ must not fall into the gap between $n_{S_1}(V)$ and $n_{S_2}(V)$ for any generated subspace $V$. To do so, we bound the gap $|n_{S_1}(V)-n_{S_2}(V)|$ for every $V$. 

While \cite{kaplan2025differentially} studied the stable partition in the context of differential privacy, i.e., analyzing its sensitivity to single-element modifications, we extend its properties to the replicability setting, where we handle the broader challenge of entirely new samples. In our key technical lemma (Lemma \ref{lem:phiwhp}), we bound the maximum gap between multiplicities obtained from two independent samples, uniformly over all subspaces. The lemma is proven in Appendix~\ref{sec:main_lem}.

\begin{lemma}\label{lem:phiwhp}
For any field $\bbF$ and any distribution $\cD$ over $\bbF^d$, the following holds for any $m \in \mathbb{N}$ and $\delta \in (0,1)$:
\[
Pr_{S_1,S_2\sim\mathcal{D}^m}\Big[\exists V\subseteq \mathbb{F}^d : |n_{S_1}(V) - n_{S_2}(V)| \geq  3\sqrt{m\ln(1/\delta)} \Big] \leq \delta
\]
\end{lemma}

To prove Lemma~\ref{lem:phiwhp}, we first employ McDiamid's inequality to prove that $\sup_{V} |n_{S_1}(V) - n_{S_2}(V)|$ is close w.h.p. (over the choices of $S_1,S_2$) to its expectation,
since $\sup_V |n_{S_1}(V) - n_{S_2}(V)|$ changes by at most $1$ when a single element of 
$S_1$ or $S_2$ is replaced, by the $\ell_\infty$-sensitivity of the stable partition. 
The main challenge lies in bounding the expectation 
$\mathbb{E}_{S_1, S_2}\left[\sup_V |n_{S_1}(V) - n_{S_2}(V)|\right]$, 
where the supremum ranges over the (possibly)  infinite collection of subspaces of $\mathbb{F}^d$. 
Moreover, unlike standard sensitivity analyses for differential privacy, which track 
the effect of a single coordinate change, here $S_1$ and $S_2$ are fully independent 
samples, so deviations can accumulate across all $m$ coordinates. To handle this, we 
use a symmetrization argument that reduces the problem to bounding a sum of variances 
of the differences $\set{n_{S_1}(V)-n_{S_2}(V)}_{V \subseteq \mathbb{F}^d}$. Then, we apply the 
Efron-Stein inequality that effectively bounds these variances via 
a bound on the $\ell_2$-sensitivity 
of the differences, which we derive by 
combining the $\ell_\infty$ and $\ell_1$ sensitivity guarantees of the stable partition.

Based on Lemma \ref{lem:phiwhp}, we can guarantee that the gaps are relatively small, and by sampling a threshold $t$ uniformly at random from $(0,T_{max}]$ for sufficiently large $T_{max}$, the probability of a uniformly drawn $t$ landing in the bounded gap between $n_{S_1}(V)$ and $n_{S_2}(V)$ becomes small, ensuring both executions agree on the candidate subspaces.
Finally, after isolating the candidate subspaces, we extract the one of maximal dimension, and show that due to the partition's nested structure, this guarantees coverage of a $1-\varepsilon$ fraction of the original input vectors.

\paragraph{From Replicable Linear Span to Parity Learning (Theorem~\ref{thm:intro:parities})
}

In realizable PAC learning of parities, given labeled examples $S = \set{(x_i,y_i)}_{i=1}^m \in (\bbF_2^d \times \bbF_2)^m$, it suffices to compute $w \in \bbF_2^d$ with $y_i = \ip{x_i,w}$ for most $i$'s (i.e., low empirical error) only when $S$ is realizable. The challenge is that replicability, by definition, should hold \emph{for any} distribution $\cD$ over $\bbF_2^d \times \bbF_2$, even for \emph{non-realizable} ones. To do that, we compute a replicable linear span $U^*$ of the $(d+1)$-dimensional points $S= \set{(x_i,y_i)}_{i=1}^m$, filter the points to $S_{U^*} = S \cap U^*$, and continue as follows: If $e_{d+1} = (0,\ldots,0,1) \in U^*$, we output $\bot$. Otherwise, we output a uniform vector in the set $C_{U^*}$ of all the $w$'s that satisfy $y_i = \ip{x_i,w}$ for every $(x_i,y_i) \in S_{U^*}$. 
In the analysis, we prove that $S_{U^*}$ is realizable iff $e_{d+1} \notin U^*$. Now, when $S$ is realizable, then $S_{U^*}$ is realizable, so the algorithm will output $w$ that consistent with all pairs in $S_{U^*}$ (which covers most of the examples in $S$). Regarding replicability, since $U^*$ is replicable, the decision of whether to output $\bot$ or not is also replicable. Given that the output is not $\bot$, then the resulting filtered datasets $S_{U^*}$ and $S_{U^*}'$ of two executions are both realizable, and span the same subspace $U^*$ with high probability. In that case, the sets of possible solutions $C_{U^*}$ and $C_{U^*}'$ must be identical, so the final output $w$ (using the same shared randomness) must also be identical.

\section{Related Work}\label{sec:related_work}

Most closely related to our work are the results of~\cite{impagliazzo2022reproducibility} and~\cite{kalavasis2024computational}, both of which include algorithms for replicably learning parities over restricted distributions. As mentioned in Section~\ref{sec:tech}, \cite{impagliazzo2022reproducibility} observes that over the uniform distribution there is a unique parity function that is consistent with the data, and so simply drawing sufficiently many samples and performing Gaussian elimination will give the same parity function with high probability over the choice of sample. Notably, this approach does not quite satisfy the definition of replicability, which is universally quantified over distributions, and Gaussian elimination may fail to be replicable on distributions far from uniform. \cite{kalavasis2024computational} gives an algorithm for lifting any replicable algorithm for the uniform distribution to general distributions, but with exponential running time in the decision tree complexity of the target distribution. So, for distributions where only a constant number of features determine the probability mass of an element, their algorithm will run in polynomial time, but may run in time exponential in $d$ for arbitrary distributions. 

In this work, we focus specifically on computational tractability of replicable learning in the supervised PAC learning setting, and adopt the strong notion of replicability introduced in~\citep{impagliazzo2022reproducibility} which requires that the algorithm returns precisely the same output with high probability over the choice of sample and shared randomness. However, replicability has been studied in other contexts,  including bandits~\citep{esfandiari2022replicable}, clustering~\citep{esfandiari2023replicable}, online learning~\citep{ahmadi2024replicable}, high-dimensional statistics~\citep{hopkins2024replicability,BBEGH26}, reinforcement learning~\citep{karbasi2023replicability,eaton2023replicable,eaton2025replicable,hopkins2025generative,zhang2025list}, active learning~\citep{hira2024cost}, and hypothesis testing~\citep{BBEGH26}. 

Related stability notions also abound in the literature. Replicability itself is a weakening of global stability~\citep{bun2020equivalence}, which also requires that an algorithm return precisely the same output with high probability over samples, but does not allow shared randomness between independent runs of the algorithm. Global stability is equivalent to the subsequently introduced notion of list replicability~\citep{chase2023stability}, which instead requires that for every distribution there exists a list of outputs such that with high probability over samples the learner returns on output from this list. Relaxations of replicability have also been recently studied. \cite{hopkins2025approximate} give two new relaxations of replicability which maintain the setup of comparing the outputs of two independent runs of the same algorithm on resampled data and shared internal randomness, but weaken the requirement of exact equality of outputs in the following ways:
\begin{itemize}
    \item pointwise replicability - output models must agree for a fixed input from the data domain, but not all data simultaneously.
    \item approximate replicability - output models must agree on a large fraction of the target distribution.
\end{itemize}
They also introduce semi-replicability, which is analogous to semi-private algorithms, in which the algorithm must return identical outputs on independent runs with shared randomness and resampled data as in replicability, but the two runs may be given shared unlabeled data in addition. They show that these relaxations overcome statistical barriers to strictly replicable learning, obtaining sample complexity guarantees for all that depend only on VC dimension, as opposed to Littlestone dimension, which is known to be required for replicable learning (via the statistical equivalence with privacy~\cite{ghazi2021user}).

\section{Preliminaries}  \label{sec:prelim}

\subsection{Notations}

We use the symbol $\bbF$ to denote a field, and 
denote by $\mathbb{F}_2$ the binary field over $\zo$ where arithmetic is performed
modulo $2$. Given two vectors $u = (u_1,\ldots,u_d) $ and $v = (v_1,\ldots,v_d)$ in $\mathbb{F}^d$, the inner product between $u$ and $v$ is $\ip{u,v} = \sum_{i=1}^d u_i v_i \in \bbF$.
Given a distribution $\cD$, we write $x \sim \cD$ to denote that $x$ is sampled according to $\cD$, and write $\cD(x) = Pr_{x' \sim \cD}[x' = x]$.
For a set $\cT$, we write $x \la \cT$ to denote that $x$ is sampled uniformly at random from $\cT$. All logarithms considered here are natural logarithms (i.e., in base $e$).

\subsection{PAC Learning}

\begin{definition}[(Realizable) PAC Learnability, see e.g., \cite{shalev2014understanding}]\label{def:PAC}
A hypothesis class $\mathcal{H} = \{h \colon \mathcal{X} \to \{0,1\}\}$ is \emph{PAC learnable} if there exist a function
$m_{\mathcal{H}} : (0,1)^2 \to \mathbb{N}$
and a learning algorithm $\mathcal{A}$ such that the following holds: for every $\varepsilon, \delta \in (0,1)$, and for every distribution $\mathcal{D}$ over $\mathcal{X} \times \{0,1\}$ that is realizable with respect to $\mathcal{H}$ (i.e., there exists $f \in \mathcal{H}$ such that the support of $\mathcal{D}$ only contains pairs of the form $(x, f(x))$), when $\mathcal{A}$ is given $m \ge m_{\mathcal{H}}(\varepsilon, \delta)$
i.i.d. examples drawn from $\mathcal{D}$, the algorithm outputs a hypothesis $h \in \mathcal{H}$ such that
\[
\underset{\substack{S \sim \mathcal{D}^m\\ h \sim \mathcal{A}(S)}}{\Pr}\big[ L_{\mathcal{D}}(h) \le \varepsilon \big] \ge 1 - \delta,
\]
where $L_{\mathcal{D}}(h) := \Pr_{(x,y) \sim \mathcal{D}}[h(x) \ne y]$ denotes the generalization error of $h$.

\end{definition}

\remove{
\begin{definition}[PAC Learnability, see e.g., \cite{shalev2014understanding}]
A hypothesis class $\mathcal{H}$ over a domain $\mathcal{X}$ is \emph{PAC learnable} if there exist a function
$m_{\mathcal{H}} : (0,1)^2 \to \mathbb{N}$
and a learning algorithm $\mathcal{A}$ such that the following holds: for every $\varepsilon, \delta \in (0,1)$, for every distribution $\mathcal{D}$ over $\mathcal{X}$, and for every labeling function $f : \mathcal{X} \to \{0,1\}$ satisfying the realizability assumption with respect to $\mathcal{H}$, when $\mathcal{A}$ is given
$m \ge m_{\mathcal{H}}(\varepsilon, \delta)$
i.i.d. examples drawn from $\mathcal{D}$ and labeled by $f$, the algorithm outputs a hypothesis $h \in \mathcal{H}$ such that
\[
\Pr_{\mathbf{S} \sim \mathcal{D}^m}\big[ L_{\mathcal{D}}(h) \le \varepsilon \big] \ge 1 
\]

where $L_{\mathcal{D}}(h) = \Pr_{\mathbf{x} \sim \mathcal{D}}[h(\mathbf{x}) \ne f(\mathbf{x})]$ denotes the generalization error of $h$.
\end{definition}
}

In this work, we focus on the class of parity functions.

\begin{definition}\label{def:parity}
    The class of parity functions over $\bbF_2^d$ is defined by
    \begin{align}
    \cH_{PARITY}^d := \set{f_z\colon \bbF_2^d \rightarrow \bbF_2 \:\: \colon \:\: z \in \bbF_2^d}\:\:\:\text{ for }\:\:\:f_z(x) := \langle z, x \rangle.
    \end{align}
\end{definition}

\subsection{Replicability}

Replicability is a stability property of randomized algorithms, requiring that two executions of the algorithm on independent samples drawn from the same distribution produce the same output with high probability, provided the internal randomness is identical.

\begin{definition}[Replicability \cite{impagliazzo2022reproducibility}]\label{def:replicability}
A randomized algorithm $\mathcal{A} : \mathcal{X}^n \to \mathcal{Y}$ is $\rho$-replicable if for every distribution $\mathcal{D}$ over $\mathcal{X}$, we have
\[
\Pr_{\mathbf{S_1},\mathbf{S_2},\mathbf{r}}[\mathcal{A}(\mathbf{S_1;\mathbf{r}}) = \mathcal{A}(\mathbf{S_2}; \mathbf{r})] \ge 1 - \rho,
\]
where $\mathbf{S_1}, \mathbf{S_2} \in \mathcal{X}^n$ are independent sequences of i.i.d.\ samples from $\mathcal{D}$, and $\mathbf{r}$ represents the coin tosses of the algorithm $\mathcal{A}$.
\end{definition}

\subsection{Differential Privacy}

\begin{definition}[$\varepsilon$-differential privacy, \cite{dwork2006calibrating}]\label{def:privacy}
A randomized algorithm $\mathcal{A}$ is $\varepsilon$-differentially private if for all neighboring databases $S, S'$, and for all sets $\mathcal{Y}$ of outputs,

$$\Pr[\mathcal{A}(S) \in \mathcal{Y}] \le \exp(\varepsilon) \cdot \Pr[\mathcal{A}(S') \in \mathcal{Y}].$$

The probability is taken over the random coins of $\mathcal{A}$.
\end{definition}

\subsection{Statistical Query  Learning}

In the statistical query (SQ) model, algorithms access statistical properties of a distribution
rather than individual examples.

\begin{definition}[Statistical query oracle]
Let $\tau \in [0, 1]$ and $\phi : \mathcal{X} \rightarrow [0, 1]$ be a query. Let $\mathcal{D}$ be a distribution over domain $\mathcal{X}$. A \emph{statistical query oracle} for $\mathcal{D}$, denoted $\mathcal{O}_{\mathcal{D}}(\tau, \phi)$, takes as input a tolerance parameter $\tau$ and a query $\phi$, and outputs a value $v$ such that $|v - \mathbb{E}_{x \sim \mathcal{D}}[\phi(x)]| \le \tau$.
\end{definition}

\subsection{Main Tools}

\subsubsection{Inequalities}

\begin{theorem}[McDiarmid's Inequality \cite{mcdiarmid1989method}]\label{McDiarmid}
Let $f : \mathcal{Z}^m \to \mathbb{R}$ be a measurable function such that there exist constants
$c_1,\dots,c_m$ satisfying
\[
\sup_{S \in \mathcal{Z}^m,\; z_i' \in \mathcal{Z}}
\big| f(S) - f(S^{(i)}) \big| \le c_i,
\]
where $S^{(i)}$ denotes the vector obtained from $S$ by replacing its $i$-th coordinate with $z_i'$. Then for every 
collection of independent random variables $\bS = (\mathbf{z_1},\ldots,\mathbf{z_m}) \in \mathcal{Z}^m$ and every $\lambda > 0$,
\[
\Pr\big[ f(\bS) - \mathbb{E}[f(\bS)] \ge \lambda \big], \:\: \Pr\big[ f(\bS) - \mathbb{E}[f(\bS)] \le -\lambda \big]
\: \le \:
\exp\!\left(
-\frac{2\lambda^2}{\sum_{i=1}^m c_i^2}
\right).
\]

\end{theorem}

\begin{theorem}[Efron-Stein Inequality \cite{efron1981jackknife}]
Let $f : \mathcal{Z}^m \to \mathbb{R}$ be a measurable function. Let $\mathbf{S} = (\mathbf{z}_1, \dots, \mathbf{z}_m) \in \mathcal{Z}^m$ be a collection of independent random variables. Let $\mathbf{S}^{(i)}$ denote the vector obtained from $\mathbf{S}$ by replacing its $i$-th coordinate with $\mathbf{z}_i'$, with $\mathbf{z}_i,\mathbf{z}_i' $ having the same distribution. Then, the variance of $f(\mathbf{S})$ satisfies:
\[
    \mathrm{Var}[f(\mathbf{S})] \le \frac{1}{2} \sum_{i=1}^m \mathbb{E} \left[ \left( f(\mathbf{S}) - f(\mathbf{S}^{(i)}) \right)^2 \right],
\]
\end{theorem}

\subsubsection{Stable Partition} \label{subsec:kmm}

Given a sequence of vectors in $\mathbb{F}^d$ (for an arbitrary field $\bbF$), the Stable Partition algorithm iteratively constructs (in a greedy manner) a collection of sets that partition the input, where each set is linearly independent.
Importantly, this procedure enjoys strong stability properties.

\begin{theorem}[Stable Partition, Algorithm~1 in \cite{kaplan2025differentially}]\label{thm:stablePartition}
    Consider the algorithm that given input $S = \{v_1, \dots ,v_m\}$ where $v_i\in\mathbb{F}^d$, the algorithm  partitions $S$ into 
\[
\mathcal{P}_S = \{A_1, A_2, \dots, A_M\},
\]
such that each set $A_j$ is linearly independent. 
    Then the partition satisfies the following key properties:

\begin{enumerate}
\item \textbf{Nested Subspaces.}  
There exist at most $d$ subspaces
\[
V_1 \supseteq V_2 \supseteq \cdots \supseteq V_\ell,
\quad \ell \le d,
\]
such that for every $j$, $\mathrm{span}(A_j)$ is equal to one of these subspaces. 

\item \textbf{Low $\ell_\infty$-Sensitivity.}  
We define the  multiplicity function over subspaces $V \subseteq \mathbb{F}^d$, 
\[
n_S(V) :=\big|\{ A_j\in \mathcal{P}_{S} : \mathrm{span}(A_j) = V \}\big|.
\]
Then, $\set{n_S(V)}_V$ has $\ell_\infty$-sensitivity at most $1$, meaning that for any neighboring input sequences $S,S'$ differing in a single vector,
\[
\max_V |n_S(V) - n_{S'}(V)| \le 1,
\]

\item \textbf{Low $\ell_1$-Sensitivity.} 

$\set{n_S(V)}_V$ has $\ell_1$-sensitivity at most $2$, meaning that for any neighboring input sequences $S,S'$ differing in a single vector,
\[
\sum_V |n_S(V) - n_{S'}(V)| \le 2,
\]
\end{enumerate}

\end{theorem}

For completeness, a full description of the algorithm is provided in Appendix  \ref{sec:stable_partition}.

\section{Replicable Linear Span} \label{sec:linspan}

 In this section, we present our replicable linear span algorithm, $\RepLinearSpan$. Given a sequence of $m$ vectors $S = (v_1, \dots, v_m) \in (\mathbb{F}^d)^m$, the algorithm operates by first splitting the sequence into linearly independent sets and then filtering them to identify a replicable set of "heavy" subspaces. Finally, it outputs a subspace of maximal dimension, $V^* \subseteq Span(S)$, that covers almost all input vectors. For the initial decomposition step, our algorithm utilizes the stable partition algorithm (Theorem~\ref{thm:stablePartition}) strictly as a subroutine.

\bigskip

\begin{algorithm}[H]
\caption{$\RepLinearSpan$ (Replicable Linear Span)}
\label{alg:linspan}
\begin{algorithmic}[1]

\REQUIRE  $S=\{v_i\}_{i=1}^m \in (\bbF^d)^m$.
\ENSURE A replicable subspace $V^* \subseteq Span(S)$.

\STATE 
Run the algorithm from Theorem~\ref{thm:stablePartition} to obtain partition 
$\mathcal{P}_S = \{A_1,\dots,A_M \}$.

\STATE For each subspace $V$ spanned by a set in the partition $\mathcal P_S$, define  
\[
n_{S}(V) := |\{ j : \mathrm{span}(A_j) = V \}|.
\]

\STATE   
Pick $t \la (0, T_{\max}]$ uniformly at random (the threshold $ T_{\max}$ will be determined later by the analysis).

\STATE Define the heavy subspaces  \label{step:Pheavy}
$
\mathcal{P}_{\text{heavy}} := \{ V : n_{S}(V) \geq t  \}.
$

\STATE Let  $V^* := \arg\max_{V \in \mathcal{P}_{\text{heavy}}} \dim(V).$

\STATE \textbf{Return} $V^*$
\end{algorithmic}
\end{algorithm}

\begin{theorem}[Restatement of Theorem~\ref{thm:intro:linearspan}]\label{thm:linspan}
Let $\varepsilon, \rho \in [0,1]$. There exists a choice of $T_{\max}$ for $\RepLinearSpan$ (Step~3 in Algorithm~1)
such that $\RepLinearSpan$, for datasets of size 
\begin{align*}
m = O\Big( \frac{d^6}{\rho^2 \varepsilon^2} \log \left(1/\rho \right)\Big),
\end{align*}
is $\rho$-replicable (Definition~\ref{def:replicability}),  and for every $S = (v_1,\ldots,v_m) \in (\mathbb{F}^d)^m$, 
the output $\mathbf{V^*} = \RepLinearSpan(S)$ satisfies with probability $1$: 
\begin{enumerate}
    \item $\mathbf{V^*}\subseteq\mathrm{span}(S)$
    
    \item  
    $\frac{|S \backslash \mathbf{V^*}|}{m} \leq \varepsilon$.
\end{enumerate}
Furthermore, the running time of $\RepLinearSpan$ is bounded by $O(m^2d^3)$.
\end{theorem}

The proof of Theorem~\ref{thm:linspan} appears in Section~\ref{sec:linear_span_proof}.

\subsection{Direct Applications}\label{sec:general-vec-spaces:applicaitons}
 We observe that the tasks of \emph{synthetic linear equations}, \emph{synthetic affine spans}, and \emph{PAC learning of subspaces} (Section~3 in \cite{kaplan2025differentially}) can all be realized as direct applications of the linear span problem (Theorem~\ref{thm:linspan}). As shown in \cite{kaplan2025differentially}, the solutions to these tasks can be derived by post-processing the output of a private linear span algorithm.

Like differential privacy, replicability is preserved under post-processing. Therefore, by applying the corresponding transformations to the output of our $\RepLinearSpan$ algorithm, we immediately obtain replicable algorithms for these three tasks. We omit the formal details as they strictly follow the derivations in \cite{kaplan2025differentially}.

\section{PAC Learning of Parities}\label{sec:learning_parities}

In this section, we show how $\RepLinearSpan$ (Algorithm~\ref{alg:linspan}) can be used as a subroutine for replicable parity learning: 
Given i.i.d. samples $(x_1,y_1),\ldots,(x_m,y_m) \in \bbF_2^d \times \bbF_2$ drawn from an (unknown) distribution $\cD$ that is realizable with respect to the class of parity functions $\cH_{PARITY}^d = \set{f_z \colon z \in \bbF_2^d}$ (Definition~\ref{def:parity}),
the algorithm, with high probability, outputs $w \in \bbF_2^d$ such that $\Pr_{(x,y) \sim \cD}[f_w(x) = y]$ is high.

\begin{algorithm}[H]
\caption{Replicable Learner for PARITY functions}
\label{alg:parity2}
\textbf{Input:} Labeled sample $S = \{(x_i, y_i)\}_{i=1}^m \in (\mathbb{F}_2^d \times \mathbb{F}_2)^m$. \\
\textbf{Output:} A parity $f_w: \mathbb{F}_2^d \to \mathbb{F}_2$, or $\bot$.
\begin{algorithmic}[1]
\STATE Compute $U^* \sim \text{RepLinearSpan}(S)$ (Algorithm \ref{alg:linspan} over $\mathbb{F}_2^{d+1}$).
\STATE Let $e_{d+1} = (0,\dots,0, 1) \in \mathbb{F}_2^{d+1}$. 
\IF{$e_{d+1} \in U^*$}
    \RETURN $\bot$ 
\ELSE
    \STATE Filter the sample: $S_{U^*} = S \cap U^*$.
    \STATE Compute the set $C_{U^*} \subseteq \mathbb{F}_2^d$ of all solutions $w$ to the linear system:
    $$ \langle w, x_i \rangle = y_i  \quad \text{for all } (x_i, y_i) \in S_{U^*} $$
    \STATE Pick $w \leftarrow C_{U^*}$ uniformly at random.
    \RETURN the function $f_w(x) = \langle w, x \rangle$.
\ENDIF
\end{algorithmic}
\end{algorithm}

\begin{theorem}[Restatement of Theorem~\ref{thm:intro:parities}]
\label{thm:parity_learner2}
For every $d \in \mathbb{N}$ and $\rho, \delta, \varepsilon \in (0,1)$, for a data set of size
$m = O\left( \frac{d^6}{\rho^2 \varepsilon^2} \log\left(\frac{1}{\rho}\right) + \frac{\log(1/\delta)}{\varepsilon} \right)$,
Algorithm \ref{alg:parity2} satisfies the following properties:
\begin{enumerate}
    \item \textbf{Replicability:} It is $\rho$-replicable for any  distribution $\mathcal{D}$ over $\mathbb{F}_2^d \times \bbF_2$.
    \item \textbf{Accuracy:} If $\mathcal{D}$ is realizable with respect to $\mathcal{H}_{PARITY}^d$, it PAC-learns the class.
\end{enumerate}
Furthermore, the running time is bounded by $O(m^2 d^3)$.
\end{theorem}

The proof of Theorem~\ref{thm:parity_learner2} appears in Appendix~\ref{sec:learning_parities_proof}.

\section{Limits on Generic Conversions to Pure Differential Privacy}
\label{sec:conversion_impossibility}

 In this section, we leverage our replicable parity learner (Theorem \ref{thm:parity_learner2}) to get a fundamental limitation in algorithmically converting replicable algorithms to pure differentially private algorithms.
 As shown in Theorem \ref{thm:parity_learner2}, our replicable parity learner achieves  optimal dependency in the failure parameter. Namely, the sample complexity scales as $O(\frac{\log(1/\delta)}{\varepsilon})$.

However, the parities pure private learner given by \cite{KLNRS11} requires $O(\frac{d \log \frac{1}{\delta}}{\varepsilon\alpha})$ samples, where $\alpha$ is the privacy paremeter. The work of \cite{georgiev2022privacy} shows that this gap cannot be closed with polynomial-time private algorithms, unless $NP=RP$. Namely, they prove that computationally efficient pure differential privacy fundamentally cannot enjoy this optimal confidence for parities:

\begin{theorem}[\cite{georgiev2022privacy}, Corollary 4.3] \label{thm:gh22}
Suppose $\text{RP} \neq \text{NP}$. Then every polynomial-time $\alpha$-DP algorithm which for any $\delta > 0$ can PAC-learn $d$-variable PARITYs to accuracy $\varepsilon$, succeeding with probability $1-\delta$, requires $m \ge \omega\left(\frac{\log(1/\delta)}{\alpha\varepsilon}\right)$ samples.
\end{theorem}

By combining Theorem \ref{thm:parity_learner2} and Theorem \ref{thm:gh22}, we derive the following impossibility result: one cannot generically and efficiently convert replicability into pure DP without paying a super-logarithmic penalty in sample complexity. 

\repdpimpos*

\begin{proof}
Assume towards contradiction that such a transformation exists. Applying it to our replicable parity learner (Algorithm 2) would yield a polynomial-time pure DP PAC learner for $d$-variable PARITY with a sample complexity scaling of $O(\log(1/\delta))$. However, by Theorem \ref{thm:gh22}, any such efficient pure DP learner strictly requires $\omega(\log(1/\delta))$ samples, unless $\text{RP} = \text{NP}$. 
\end{proof}

\section*{Acknowledgments}

The authors would like to thank Russell Impagliazzo, Toniann Pitassi, Sarit Kraus, Ariel Vetzler, and Mark Bun for useful discussions. 

\bibliographystyle{alpha}
\bibliography{refs}

\appendix


\section{Proof of Theorem \ref{thm:linspan}}\label{sec:linear_span_proof}

In the following, we prove that $\RepLinearSpan$ returns a replicable subspace that contains most of the input sequence.

\begin{remark}
We assume the subspace output by Algorithm \ref{alg:linspan} is represented in a canonical form (e.g., via the Reduced Row Echelon Form basis).
\end{remark}

\begin{proof}[Proof of Theorem~\ref{thm:linspan}]

We set $T_{\max} = \frac{12d}{\rho}\sqrt{m \log (2/\rho)}$ and show this choice satisfies replicability and coverage. 

\paragraph{Replicability.}

Let $\cD$ be a distribution over $\mathbb{F}^d$, let $\mathbf{S_1}, \mathbf{S_2}$ be independent datasets, each consisting of $m$ i.i.d.\ samples from $\cD$, let $\bt \leftarrow (0,T_{max}]$, and let $\bV^*_i = \RepLinearSpan(\bS_i;\bt)$. Our goal is to prove that $Pr[\bV^*_1 = \bV^*_2] \geq 1-\rho$.

The output $\mathbf{V}^*$ of the algorithm is a deterministic function of the heavy set $\mathbf{P}_{heavy}$, computed in step 4 of the algorithm. 
Therefore, to prove that $\mathbf{V}^*_1 = \mathbf{V}^*_2$, it is sufficient to show that both executions compute the same set, i.e., that $\cP_{heavy}(\mathbf{S}_1; \mathbf{t}) = \cP_{heavy}(\mathbf{S}_2; \mathbf{t})$ with high probability, where $\cP_{heavy}(S;t)$ denotes the value of $\cP_{heavy}$ (computed at Step~\ref{step:Pheavy} of $\RepLinearSpan$) given input $S$ and a fixed threshold $t$.  

Let $\bU = \{V \subseteq \mathbb{F}^d : n_{\mathbf{S}_1}(V) > 0 \text{ or } n_{\mathbf{S}_2}(V) > 0\}$. By the nested subspaces property (see Theorem \ref{thm:stablePartition}), each sample generates at most $d$ distinct subspaces. Therefore, it holds with probability $1$ that $|\bU| \le 2d$. Becasue we take $\mathbf{t}>0$, both executions will exclude all $V \notin \bU$ from $\mathcal{P}_{heavy}$ sets. Therefore, the sets $\cP_{heavy}(\mathbf{S}_1; \mathbf{t})$ and $\cP_{heavy}(\mathbf{S}_2; \mathbf{t})$ will differ if and only if the shared threshold $\mathbf{t}$ falls strictly between $n_{\mathbf{S}_1}(V)$ and $n_{\mathbf{S}_2}(V)$ for some subspace $V \in \bU$.

We next use our key technical lemma -  Lemma~\ref{lem:phiwhp}.

\bigskip
By Lemma \ref{lem:phiwhp}, with probability at least $1-\rho/2$ over the choice of $\mathbf{S}_1$ and $\mathbf{S}_2$, the gap is bounded globally:

\[ \sup_{V \subseteq \mathbb{F}^d} |n_{\bS_1}(V) - n_{\bS_2}(V)| \le 3\sqrt{m \log(2/\rho)}.
\]

Conditioned on this event, the combined length of the bad zones is at most $2d  \big(3\sqrt{m \log(2/\rho)}\big)=6d\sqrt{m \log(2/\rho)}$. Therefore, the probability that $\mathbf{t}$, which is independent of $\bU$, lands in any bad zone is simply
\[
\frac{6d\sqrt{m \log(2/\rho)}}{T_{max}} = \frac{6d\sqrt{m \log(2/\rho)}}{\frac{12d}{\rho}\sqrt{m \log(2/\rho)}} =\rho/2
\]

Applying a union bound over the two failure events, the total probability that replicability fails is bounded by $\rho/2 + \rho/2 = \rho$

\paragraph{Coverage.}

Let $S = \set{v_1,\ldots,v_m} \in \mathbb{F}^d$, and let $\bV^* = \RepLinearSpan(S)$. 
Observe that item 1 immediately follows since $\bV^*$ is spanned by a subset of the input vectors. 

The proof of item 2 follows by the following claim, proven in Section~\ref{sec:proving-clm:spancover}.

\begin{claim} \label{clm:spancover}
If $T_{\max} \leq \frac{12d}{\rho}\sqrt{m \log(2/\rho)}$ and $m \geq \frac{144d^6}{\rho^2\varepsilon^2}\log(\frac{2}{\rho})$, 

then $\Pr[\frac{|S \setminus \bV^*|}{m} \leq \varepsilon] = 1$.
\end{claim}

Note that the conditions of the claim are indeed satisfied. Specifically, the first condition holds by our explicit choice of $T_{max}$, and the second condition is satisfied by the sample complexity $m$ established in the theorem statement. Thus, the coverage guarantee holds with probability $1$.

\paragraph{Running-time.}

The Stable Partition algorithm (see Appendix \ref{sec:stable_partition}) performs at most $O(m^2)$ linear independence tests in total. Each test can be implemented in polynomial time (e.g., via Gaussian elimination). Therefore, the total running time of Stable Partition is
$O(m^2 d^3)$.

In Step~2 of Algorithm \ref{alg:linspan}, the algorithm considers only the subspaces spanned
by the sets in the partition $\mathcal{P}_{S}$.
By the nested subspaces property, there are at most $d$ distinct such subspaces.
Computing a canonical representation of each subspace and iterating over the
partition can be done in $O(m d^3)$ time.

The remaining steps, sampling the threshold $t$, identifying the heavy
subspaces, and selecting the maximum-dimension subspace, can all be performed
in $O(d)$ time.
Overall, Algorithm \ref{alg:linspan} runs in time $O(m^2 d^3)$.

\end{proof}

\subsection{Proving Claim~\ref{clm:spancover}}\label{sec:proving-clm:spancover}

In the following, recall that we fixed a dataset $S \in (\mathbb{F}^d)^m$, that $\cP_S$ denotes the stable partition of $S$ into linearly independent sets $\set{A_1,\ldots,A_M}$ (according to Section~\ref{subsec:kmm}), and that for every subspace $V\subseteq \mathbb{F}^d$ we denote $n_{S}(V) = |\{ j : \mathrm{span}(A_j) = V \}|$. Let $\bP_{heavy}$ and $\bV^*$ be the values of $\cP_{heavy}$ and $V^*$ in a random execution of $\RepLinearSpan(S)$.

The proof of Claim~\ref{clm:spancover} uses the following claim.

\begin{claim} \label{clm:pigeonhole}
If $T_{\max} < m/d^2$, then $\bP_{\mathrm{heavy}}\neq \emptyset$ with probability $1$.
\end{claim}

\begin{proof}
Each $A_j$ is linearly independent, hence $|A_j|\le d$, so
$|\mathcal{P}_S|\ \ge\ \frac{m}{d}.$
Moreover, the algorithm from Theorem~\ref{thm:stablePartition} produces a nested chain of subspaces, so the number of distinct subspaces is at most $d$. By the pigeonhole principle, there exists a subspace $V$ with
$n_{S}(V)\ \ge\ \frac{|\mathcal{P}_S|}{d}\ \ge\ \frac{m}{d^2}.$
Then for $T_{\max}<\frac{m}{d^2}$ it must be that $\bP_{\text{heavy}}\neq\emptyset$.
\end{proof}

We are now ready to prove Claim~\ref{clm:spancover} using Claim~\ref{clm:pigeonhole}.

\begin{proof}[Proof of Claim~\ref{clm:spancover}]
Recall that $\bV^*$ is the subspace of maximum dimension among \(\bP_{\text{heavy}}\), and let $\mathbf{k}=dim(\bV^*)$ (by Claim \ref{clm:pigeonhole}, $\bV^*$ is well-defined and $\mathbf{k}>0$). By the nested property of the subspaces, 
each input vector outside $\bV^*$ must belong to an independent set spanning subspace $V'$ of dimension larger than $\mathbf{k}$.  

Each set spanning a subspace $V'$ with $dim(V')>\mathbf{k}$ is of size at most $d$. there are at most $d-\mathbf{k} < d$ such subspaces. By the fact that these subspaces are not in  \(\bP_{\text{heavy}}\), we know that each one of them is spanned by at most $\mathbf{t}\le T_{max}$ sets. Therefore $|S\setminus \bV^*| \leq d^2 T_{max} $.

The fraction of vectors outside $\bV^*$ is at most 

\begin{align*}
   \frac{|S \setminus V^*|}{m} \\
&\leq \frac{d^2 T_{max}}{m} \\
&\leq \frac{d^2}{m} \left( \frac{12d}{\rho}\sqrt{m \log(2/\rho)} \right) \\ &= \frac{12d^3}{\rho}\sqrt{\frac{\log(2/\rho)}{m}} \\
&\leq \frac{12d^3}{\rho}\sqrt{\frac{\rho^2\varepsilon^2}{144d^6}} = \varepsilon
\end{align*}

which concludes the proof of the claim.
\end{proof}

\section{Proof of Theorem  \ref{thm:parity_learner2}}\label{sec:learning_parities_proof}

To prove the correctness of Algorithm 2, we first prove the following claim. 

\begin{claim}
\label{clm:condition_realizability}
Let $S = \{(x_i, y_i)\}_{i=1}^m \subseteq \mathbb{F}_2^d \times \mathbb{F}_2$ be an arbitrary labeled sample\remove{, and let $S_{ext} = \{(x_i, y_i)\}_{i=1}^m \subseteq \mathbb{F}_2^{d+1}$ be its extended representation}. Let $e_{d+1} = (0,\ldots,0, 1) \in \mathbb{F}_2^{d+1}$. The sample $S$ is realizable by some parity function $f_w(x) = \langle w, x \rangle$ if and only if $e_{d+1} \notin \text{Span}(S)$\remove{$e_{d+1} \notin \text{span}(S_{ext})$}.
\end{claim}

\begin{proof}
Assume $S$ is realizable by some target parity $w \in \mathbb{F}_2^d$. This implies that for every $(x_i, y_i) \in S$, we have $\langle w, x_i \rangle = y_i$, which is  equivalent to $\langle (w, 1), (x_i, y_i) \rangle = 0$. By the linearity, $(w, 1)$ must be orthogonal to every vector $v \in \text{Span}(S)$. 

Assume towards contradiction that $e_{d+1} \in \text{Span}(S)$. Then $(w, 1)$ must be orthogonal to $e_{d+1}$. However, evaluating the inner product yields $\langle (w, 1), (\mathbf{0}, 1) \rangle = \langle w, \mathbf{0} \rangle + 1 \cdot 1 = 1 \neq 0$. This is a contradiction. Therefore, $e_{d+1} \notin \text{Span}(S)$.

For the other direction, Assume $e_{d+1} \notin \text{Span}(S)$. Let $V = \text{span}(S)$. In finite-dimensional linear algebra, a vector is excluded from a subspace if and only if there exists some vector in the orthogonal complement $V^\perp$ that is not orthogonal to it. In other words, $e_{d+1} \notin (V^\perp)^\perp$. 

This guarantees the existence of some vector $z \in V^\perp$ such that $\langle z, e_{d+1} \rangle \neq 0$. Since we operate over $\mathbb{F}_2$, this inner product must evaluate to exactly $1$, meaning the last coordinate of $z$ must be $1$. We can therefore write $z = (w, 1)$ for some $w \in \mathbb{F}_2^d$. Because $z \in V^\perp$, for every $(x_i, y_i) \in S$, we have $\langle (w, 1), (x_i, y_i) \rangle = 0 \implies \langle w, x_i \rangle = y_i$. Thus, $S$ is realizable by $w$.
\end{proof}

\begin{proof}[Proof of Theorem \ref{thm:parity_learner2}]
\textbf{Replicability} 
Consider two independent executions of Algorithm \ref{alg:parity2} on datasets $S_1, S_2 \sim \mathcal{D}^m$ drawn from an arbitrary distribution $\mathcal{D}$, using shared internal randomness. 

By Theorem \ref{thm:linspan}, \text{RepLinearSpan} guarantees that with probability at least $1-\rho$, both executions output identical linear subspaces: $U_1^* = U_2^* = U^*$. We condition on this event.

 If $e_{d+1} \in U^*$, both  return $\bot$. If $e_{d+1} \notin U^*$, both executions filter their datasets to obtain $S_{1,U^*},S_{2,U^*}$ with $\text{Span}(S_{1,U^*})=\text{Span}(S_{2,U^*})=U^*$.

Because $e_{d+1} \notin U^*$, by claim \ref{clm:condition_realizability}, this guarantees that $S_{1, U^*},S_{2, U^*}$ are  realizable. Thus, there's a vector $w$, s.t. for every $(x_i, y_i)$, $\langle w, x_i \rangle = y_i$. This means there's a vector $z=(w,1)\in \mathbb{F}_2^{d+1}$ s.t. $\langle z, (x_i,y_i) \rangle =0$ for all $i$. Then, solution set computed in step 8 is exactly 
\[
C_{U^*} = \big\{ w\in\mathbb{F}_2^d : (w,1) \in (U^*)^\perp  \big\}
\]

The set $C_{U^*}$ of consistent parities is not empty, and identical in both executions.

\vspace{2mm}
\noindent \textbf{Accuracy} 
Assume $\mathcal{D}$ is realizable by some target parity $w^*$. By claim \ref{clm:condition_realizability}, this guarantees $e_{d+1} \notin \text{span}(S_{ext})$. 

Let $U^*$ be the subspace returned by \text{RepLinearSpan}. Since $U^* \subseteq \text{span}(S_{ext})$, it is impossible for $e_{d+1}$ to reside in $U^*$. Therefore, Algorithm \ref{alg:parity2} does not return $\bot$ in this case. 

By Theorem \ref{thm:linspan}, provided $m = O\left( \frac{d^6}{\rho^2 \varepsilon^2} \log\left(\frac{d}{\rho}\right) \right)$, the subspace $U^*$ covers at least a $1 - \varepsilon/2$ fraction of the sample, so $|S_{U^*}| \ge m(1 - \varepsilon/2)$. The algorithm selects a hypothesis $w \in C_{U^*}$, which by definition correctly classifies every point in $S_{U^*}$. Thus, the empirical error of $f_w$ on the full sample $S$ is at most $\varepsilon/2$.

By the multiplicative Chernoff bound and a union bound over $\mathcal{H}_{PARITY}^d$, the probability that a hypothesis with true generalization error $L_{\mathcal{D}}(f) > \varepsilon$ achieves an empirical error $L_S(f) \le \varepsilon/2$ is bounded by $2^d \exp(-\varepsilon m / 8)$. Given a sample of size $m = O\left(\frac{d + \log(1/\delta)}{\varepsilon}\right)$,
with probability at least $1-\delta$,
no hypothesis with a true error greater than $\varepsilon$ will have an empirical error less than or equal to $\frac{\varepsilon}{2}$.
Since our algorithm produces a hypothesis $f_w$ with $L_S(f_w) \leq \frac{\varepsilon}{2}$, we have
\[
L_{\mathcal{D}}(f_w) \leq \varepsilon
\]

\vspace{2mm}
\noindent \textbf{Running Time:} \\
Algorithm 2 first invokes \text{RepLinearSpan} over dimension $d+1$, which runs in time $O(m^2 d^3)$. Checking whether $e_{d+1} \in U^*$ requires  $O(d^3)$ operations. Filtering the dataset takes $O(md)$ operations. Solving a system of at most $m$ linear equations in $d$ variables over $\mathbb{F}_2$ takes $O(md^2)$ time via Gaussian elimination. Finally, sampling a solution requires $O(d^2)$ time. Overall, the runtime is bounded by $O(m^2 d^3)$.
\end{proof}


\section{Proof of Lemma \ref{lem:phiwhp}} \label{sec:main_lem}

\begin{proof}
    Combining Claims~\ref{clm:concentration_phi} and \ref{clm:expectation_phi} (stated and proven below), we get that w.p.\ at least $1-\delta$
    \begin{align*}
    \sup_{V}|n_{S_1}(V)-n_{S_2}(V)| \\
    &\leq\mathbb{E}\Big[\sup_{V}|n_{S_1}(V)-n_{S_2}(V)|\Big]+\sqrt{m \log{1/\delta}} \\
        & \leq \sqrt{4m}+\sqrt{m \log{1/\delta}} \\
        &\leq 3\sqrt{m\log{1/\delta}}.
    \end{align*}
\end{proof}

\bigskip

Next, we define the function $\Phi(S_1, S_2)$ which we will use in the following proofs. 

\begin{definition}
Let $S_1, S_2$ be two independent samples drawn from the distribution $\mathcal{D}$ over $\mathbb{F}^d$.  Let $n_S(V)$ denote the multiplicity of $V$. We define the function $\Phi(S_1, S_2)$ as the maximum deviation in the multiplicity of any subspace $V$ across the two samples: 
\[
\Phi(S_1,S_2) = \sup_{V}|n_{S_1}(V)-n_{S_2}(V)| 
\]
\end{definition}

\begin{claim} \label{clm:concentration_phi}
 with probability $1-\delta$ over the  samples $S_1,S_2\sim\mathcal{D}^{m}$ , 
\[
\Phi \leq \mathbb{E}\Phi + \sqrt{m \log(1/\delta)}. 
\]   
\end{claim}

\begin{proof}
Let $S=(S_1,S_2)$ and $S'=(S_1',S_2)$ be 2 neighboring samples.
 By the triangle inequality:
\begin{align*}
    |\Phi(S) - \Phi(S')|  \\
    &=  \big|\sup_{V}|n_{S_1}(V)-n_{S_2}(V)| - \sup_{V}|n_{S_1'}(V)-n_{S_2}(V)| \big|\\
    &\leq \sup_{V} \big||n_{S_1}(V)-n_{S_2}(V)| - |n_{S_1'}(V)-n_{S_2}(V)| \big| \\
    &\leq \sup_{V} \big| n_{S_1}(V) - n_{S_1'}(V)\big| \\
    &\leq 1
\end{align*}

Then, by McDiarmid, we have 
\[
Pr[\Phi \geq \mathbb{E}\Phi + \sqrt{m \log(1/\delta)}] \leq \delta. 
\]
(Note that because the function $\Phi(S_1, S_2)$ is symmetric with respect to $S_1$ and $S_2$, an identical proof holds if the neighboring samples differ in an element of $S_2$ rather than $S_1$.)
\end{proof}

\begin{claim} \label{clm:expectation_phi}
      For any vector space $\mathbb{F}^d$, we have  
    \[   
\mathbb{E}_{S_1,S_2\sim\mathcal{D}^m}[\Phi(S_1,S_2)]  \leq \sqrt{4m}
    \]  
\end{claim}
\begin{proof}
We begin by bounding the supremum with the $L_2$ norm and applying Jensen's inequality:
   \begin{align*}   \mathbb{E}_{S_1,S_2}[\Phi(S_1,S_2)] \\&=\mathbb{E}_{S_1,S_2}\Big[\sup_{V}|n_{S_1}(V) -n_{S_2}(V)|  \Big] \\
       &\leq  \mathbb{E}_{S_1,S_2}\Big[\sqrt{\sum_V(n_{S_1}(V) -n_{S_2}(V))^2}  \Big] \\
       &\leq \sqrt{\mathbb{E}_{S_1,S_2}\Big[\sum_V(n_{S_1}(V) -n_{S_2}(V))^2  \Big]} \\
   \end{align*}

Next, we introduce a sequence of $m$ paired variables $T = \{(t_1, t_1'), \dots, (t_m, t_m')\}$ drawn i.i.d. from $\mathcal{D}$, and a sequence of independent fair coins $\sigma \in \{-1, 1\}^m$. We construct $S_1$ and $S_2$ by assigning $t_i$ to $S_1$ and $t_i'$ to $S_2$ if $\sigma_i = 1$, and swapping them if $\sigma_i = -1$. Given $T,\sigma$ the sequences $S_1,S_2$ are completely determined. We can rewrite the expectation:
\begin{align*} 
\sqrt{\mathbb{E}_{S_1,S_2}\Big[\sum_V(n_{S_1}(V) -n_{S_2}(V))^2  \Big]} \\
&=\sqrt{\mathbb{E}_{T,\sigma}\Big[\sum_V(n_{S_1}(V) -n_{S_2}(V))^2  \Big]} 
\end{align*}
For a fixed $T$, there are at most $2^m$ possible pairs of sequences, and hence a finite number (at most $d2^{m+1}$)  subspaces for which the $n_S(V)$ is non-zero. we denote these subspaces by $A(T)$.

\begin{align*}
\sqrt{\mathbb{E}_{T,\sigma}\Big[\sum_V(n_{S_1}(V) -n_{S_2}(V))^2  \Big]} \\
&=\sqrt{\mathbb{E}_{T}\Big[\sum_{V\in A(T)}\mathbb{E}_{\sigma}[(n_{S_1}(V) -n_{S_2}(V))^2]  \Big]} \\
\end{align*}

Next, we define $g_V(\sigma) = n_{S_1}(V) - n_{S_2}(V)$. Note that by the symmetric construction, $\mathbb{E}_\sigma[g_V(\sigma)] = 0$ (for any fixing of $T$), meaning $\mathbb{E}_\sigma[g_V(\sigma)^2] = 
\mathbb{E}_\sigma[(g_V(\sigma)-\mathbb{E}g_V(\sigma))^2]=\mathrm{Var}[g_V(\sigma)]$. We have,
\begin{align}\label{eq:E_T_sum_V}
=\sqrt{\mathbb{E}_{T}\Big[\sum_{V\in A(T)}\mathbb{E}_{\sigma}[(n_{S_1}(V) -n_{S_2}(V))^2]  \Big]}
&=\sqrt{\mathbb{E}_{T}\Big[\sum_{V\in A(T)}\mathbb{E}_{\sigma}[(g_V(\sigma))^2]  \Big]} \nonumber\\
&=\sqrt{\mathbb{E}_{T}\Big[\sum_{V\in A(T)}\mathrm{Var}[g_V(\sigma)]  \Big]} 
\end{align}

By Efron-Stein inequality, the variance is bounded by the expected sum of squared differences when flipping a single coin $\sigma_i$ (yielding $\sigma^{(i)}$):
\[ 
\mathrm{Var}[g_V(\sigma)]  \le \frac{1}{2} \sum_{i=1}^m \mathbb{E}_\sigma \left[ (g_V(\sigma) - g_V(\sigma^{(i)}))^2 \right] 
\]

Summing this over all $V \in A(T)$ and applying linearity of expectation:

\begin{align}\label{eq:expect-var-sum}
    \mathbb{E}_T \left[ \sum_{V \in A(T)} Var[g_V(\sigma)] \right]
    &\le  \mathbb{E}_T \left[ \sum_{V \in A(T)} \sum_{i=1 }^m \mathbb{E}_{\sigma}\left[(g_V(\sigma) - g_V(\sigma^{(i)}))^2 \right]\right]\nonumber\\
    &\leq \mathbb{E}_T \left[ \frac{1}{2} \sum_{i=1}^m \mathbb{E}_\sigma \left[ \sum_{V \in A(T)} (g_V(\sigma) - g_V(\sigma^{(i)}))^2 \right] \right]
\end{align}

We can bound the innermost sum of squares using the $L_\infty$ and $L_1$ properties:
\[
 \sum_{V} (g_V(\sigma) - g_V(\sigma^{(i)}))^2 \le \max_{V} |g_V(\sigma) - g_V(\sigma^{(i)})| \cdot \sum_{V} |g_V(\sigma) - g_V(\sigma^{(i)})| 
 \]
Flipping $\sigma_i$ changes exactly one element in $S_1$ and one in $S_2$. By the triangle inequality and $L_\infty$ sensitivity:
\begin{align*}
    \max_{V} |g_V(\sigma) - g_V(\sigma^{(i)})| \\
    &= \max_{V} \Big|\big(n_{S_1}(V)-n_{S_2}(V)\big) -  \big(n_{S_1^{(i)}}(V) -n_{S_2^{(i)}}(V)  \big)   \Big| \\
    &= \max_{V} \Big|\big(n_{S_1}(V)-n_{S_1^{(i)}}(V)\big) + \big(n_{S_2^{(i)}}(V) -n_{S_2}(V)  \big)   \Big|\\
    &\le \max_{V} |n_{S_1}(V) - n_{S_1^{(i)}}(V)| + \max_{V} |n_{S_2}(V) - n_{S_2^{(i)}}(V)|
\end{align*}

Similarly, by the $L_1$ sensitivity constraint:
\begin{align*}
   \sum_{V} |g_V(\sigma) - g_V(\sigma^{(i)})| \\
    &= \sum_{V}  \Big|\big(n_{S_1}(V)-n_{S_2}(V)\big) -  \big(n_{S_1^{(i)}}(V) -n_{S_2^{(i)}}(V)  \big)   \Big| \\
    &= \sum_{V}  \Big|\big(n_{S_1}(V)-n_{S_1^{(i)}}(V)\big) + \big(n_{S_2^{(i)}}(V) -n_{S_2}(V)  \big)   \Big|\\
    &\le \sum_{V}  |n_{S_1}(V) - n_{S_1^{(i)}}(V)| + \sum_{V}  |n_{S_2}(V) - n_{S_2^{(i)}}(V)|
\end{align*}

Applying the $L_{\infty}$  and $L_{1}$ sensitivity properties:

\[
\sum_V |n_{S}(V)-n_{S'}(V)|\leq 2, 
\] and
\[
\max_V |n_{S}(V)-n_{S'}(V)|\leq 1,
\]
We get 
\begin{align*} 
 &\sum_{V} (g_V(\sigma) - g_V(\sigma^{(i)}))^2 \\
 &\le \max_{V} |g_V(\sigma) - g_V(\sigma^{(i)})| \cdot \sum_{V \in A(T)} |g_V(\sigma) - g_V(\sigma^{(i)})| \\
 &\leq 2\cdot4 =  8
\end{align*}

Substituting this bound back into our expectation in Equation~\ref{eq:expect-var-sum} yields:

\begin{align*}
    \mathbb{E}_T \left[ \sum_{V \in A(T)} Var[g_V(\sigma)] \right] \leq  \mathbb{E}_T \left[ \frac{1}{2} \sum_{i=1}^m \mathbb{E}_\sigma [ 8 ] \right]  = \frac{1}{2} \cdot 8m  = 4m,
\end{align*}
and overall, we obtain the following bound (by the equations up to \ref{eq:E_T_sum_V}):

\[
\mathbb{E}_{S_1,S_2\sim\mathcal{D}^m}[\Phi(S_1,S_2)] \leq \sqrt{ \mathbb{E}_T \left[ \sum_{V \in A(T)} Var[g_V(\sigma)] \right] } \leq \sqrt{4m} 
\]

\end{proof}

\section{Proof of Proposition~\ref{prop:span_instability}} \label{sec:adversarial_distribution}

\begin{proof}
We partition the standard basis $\{e_1, \dots, e_d\}$ of $\mathbb{F}_2^d$ into $\sqrt{d}$ disjoint blocks $B_1, \dots, B_{\sqrt{d}}$, each containing exactly $\sqrt{d}$ vectors. We define the distribution $\mathcal{D}$ as follows: For each block index $j \in \{1, \dots, \sqrt{d}\}$ and for each basis vector $e_k \in B_j$, we assign the probability mass
\[
    \mathcal{D}(e_k) = \frac{1}{2^j \sqrt{d}}.
\]
The remaining probability mass is assigned to the zero vector.

Given the sample size $m$, we identify the "unstable block" :  $j^* = \lfloor \log_2(m / \sqrt{d}) \rfloor + 1$. 

For any vector $e_k \in B_{j^*}$, we obtain:
\[
    \frac{1}{m} \le \mathcal{D}(e_k) < \frac{2}{m}.
\]

Let $S_1, S_2 \sim \mathcal{D}^m$ be two independent samples. For any basis vector $e_k \in B_{j^*}$, the probability that it is completely missing from a single sample of size $m$ is $q=(1 - \mathcal{D}(e_k))^m$. We can bound $q$ by absolute constants strictly between $0$ and $1$:
\[
    e^{-3} \le q \le e^{-1} ,
\]
The probability that $e$ is observed in the sample is $p = 1 - q$, which is also bounded strictly by constants away from $0$ and $1$.

Because the support of $\mathcal{D}$ consists solely of standard basis vectors (and the zero vector), a standard basis vector $e_k$ is contained in the linear span of a sample if and only if $e_k$ is actually present in the sample. Namely, $e_k \in \text{span}(S)$ if and only if $e_k \in S$.

The sum of the two spans is generated by the union of their observed basis vectors, and their intersection is generated by the intersection of their observed basis vectors. Thus, we're interested in the number of basis vectors present in exactly one of the two samples:
\[
    \dim(\text{span}(S_1) + \text{span}(S_2)) - \dim(\text{span}(S_1) \cap \text{span}(S_2)) = |S_1 \Delta S_2|.
\]

A basis vector $e_k \in B_{j^*}$ will belong to the symmetric difference $S_1 \Delta S_2$ if $e_k \in S_1$ and  in $e_k \notin S_2$, or vice versa. The probability of this event for a single vector is:
\[
    \Pr[e_k \in S_1 \Delta S_2] = 2pq.
\]

Because both $p$ and $q$ are bounded by constants, $2pq \ge \gamma > 0$ for some absolute constant $\gamma$. 

Since there are $\sqrt{d}$ vectors in the block $B_{j^*}$, we define the random variable $X$ as the total number of basis vectors in $B_{j^*}$ that belong to $S_1 \Delta S_2$. By linearity of expectation,
\[
    \mathbb{E}[X] = \sum_{e_k \in B_{j^*}} 2pq \ge \gamma \sqrt{d}. 
\]

And by Markov 
\[
\Pr\left[X \geq \frac{\gamma}{2}\sqrt{d}\right] \geq \frac{\gamma}{2-\gamma}
\]

Finally, observe that the two spans are identical if and only if $X = 0$.  We can bound this probability as:

\[
 \Pr[X = 0] = (1 - 2pq)^{\sqrt{d}} \le (1 - \gamma)^{\sqrt{d}}
\]

 This is upper bounded by $e^{-\gamma\sqrt{d}}$, which is $\exp(-\Omega(\sqrt{d}))$, as required.
\end{proof}
\section{Empirical Evaluation} \label{sec:empirical}

In this section, we evaluate the performance of our Replicable Parity Learner against naive ERM (Gaussian Elimination) over $\mathbb{F}_2^d$. 
All experiments  were tested on a MacBook Air with an Apple M4 chip and 16GB RAM.

While Gaussian Elimination efficiently solves parity learning in the standard PAC model, Proposition \ref{prop:span_instability} suggest that it becomes highly unstable under biased distributions, leading to a failure in replicability. Our empirical evaluation aims to visualize this instability, and to demonstrate that our algorithm efficiently guarantees replicability, clearly illustrating the trade-off with generalization accuracy.

\subsection{Experimental Setup}
We fix the dimension at $d = 20$ and the replicability parameter at $\rho = 0.05$, and evaluate performance across varying sample sizes $m$ ranging from $50$ up to $1000$. To  measure replicability, we evaluate the algorithms over 100 independent experiments. Each experiment consists of exactly two independent trials.

 For each trial, we draw a fresh, independent training set. However, within a given experiment we sample the Replicable Learner's internal randomness (the threshold $t$) once and freeze it across both runs. Replicability is measured by checking whether the two independent trials produce the exact same hypothesis. We report the empirical  probability (the average success rate of these pairwise matches) across the 100 experiments.

To ensure a fair comparison, both algorithms  output the lexicographically smallest parity vector consistent with their learned subspaces. Generalization accuracy is evaluated on a static holdout set of 5000 samples. For our data distribution, we utilize a simplified version of the  distribution established in Proposition \ref{prop:span_instability}. Specifically, the probability mass for the standard basis vectors decays exponentially, such that $\mathcal{D}(e_i) = 2^{-i}$ for all $1 \leq i \leq d$ (the rest of the probability mass is assigned to the $0$ vector).

\subsection{Results}
The results of our empirical evaluation, illustrated in Figure \ref{fig:replicability_plot}. The right panel highlights the instability of Gaussian Elimination under the biased distribution. Because Gaussian Elimination fits the rare tail vectors present in any given sample, its learned subspace fluctuates wildly. Consequently, even as the sample size increases, independent executions of Gaussian Elimination rarely yield the exact same hypothesis, resulting in an empirical replicability score  between $20\%$ and $30\%$. In contrast, our Replicable Learner successfully forces independent runs to converge on the same hypothesis space, achieving an empirical replicability score of approximately $90\%$ at larger sample sizes. The left panel clearly illustrates the trade-off with accuracy. While Gaussian Elimination consistently reaches nearly $100\%$ test accuracy, the replicable learner algorithm achieves approximately $88\%$ at $m=1000$.

Finally, to improve efficiency in our experiments, we implement practical optimizations compared to the  theoretical requirements of Theorem \ref{thm:parity_learner2}. First, we use smaller sample sizes (and a proportionally scaled threshold). Second, because our synthetic data generation guarantees realizable labels, we omit the realizability check, as it evaluates to false with probability 1. While our results suggest that the algorithm performs well in practice under these relaxed conditions, the $O(d^6)$ dependency in sample complexity remains a limitation for deploying our methods in extremely high-dimensional regimes. 

\begin{figure}[htbp]
    \centering
    \includegraphics[width=\linewidth]{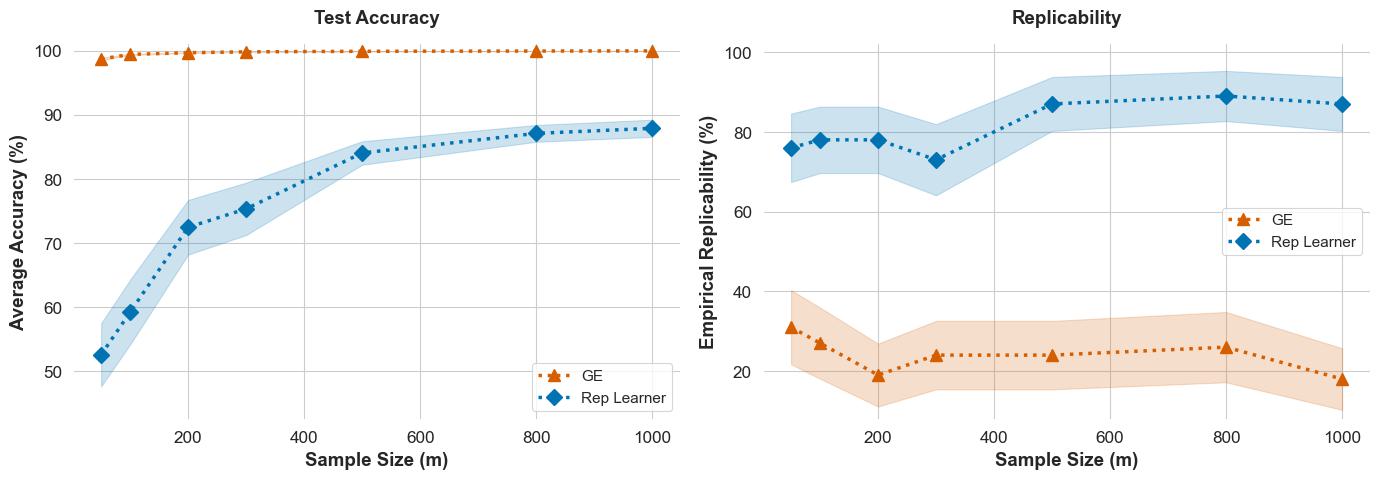}
    \caption{Accuracy and Replicability  comparing Gaussian Elimination and the Replicable Parity Learner. Shaded regions represent $\pm 2$  standard error of the mean (SEM) across independent experiments.}
    \label{fig:replicability_plot}
\end{figure}

\section{Stable Partition} \label{sec:stable_partition}

The subroutine \texttt{Stable Partition} previously appeared as Algorithm 1 in \cite{kaplan2025differentially}. We include it here for completeness.  \texttt{Stable Partition}  gets a sequence of vectors as
input and outputs a partition of the input sequence into multiple sets, each of which is linearly
independent. The key properties of the procedure 
are summarized in Theorem \ref{thm:stablePartition}
.

\begin{algorithm}[H]
\caption{Stable Partition to Linearly Independent Sets}
\label{alg:stable-partition}
\begin{algorithmic}[1]
\REQUIRE A sequence of vectors $\{v_i\}_{i=1}^m \subseteq \mathbb{F}^d$, where $v_i \neq 0$ for all $i$
\ENSURE A partition of $\{v_i\}$ into sets $\{A_j\}$, each of which is linearly independent

\STATE Initialize an empty list $\mathcal{P}$ to store the independent sets
\STATE Initialize $i = 1$
\WHILE{the input sequence is not empty}
    \STATE Initialize an empty set $A_i$
    \FOR{each vector $v_j$ in the input sequence} 
        \IF{$A_i \cup \{v_j\}$ is linearly independent}
            \STATE Add $v_j$ to $A_i$
        \ENDIF
    \ENDFOR
    \STATE Add $A_i$ to the list of independent sets $\mathcal{P}$
    \STATE Remove the vectors in $A_i$ from the input sequence
    \STATE $i \gets i + 1$
\ENDWHILE
\RETURN $\mathcal{P}$
\end{algorithmic}
\end{algorithm}

\end{document}